\documentclass[sn-basic]{sn-jnl}
\usepackage{textcomp}

\usepackage{graphicx}%
\usepackage{multirow}%
\usepackage{amsmath,amssymb,amsfonts}%
\usepackage{amsthm}%
\usepackage{mathrsfs}%
\usepackage[title]{appendix}%
\usepackage{xcolor}%
\usepackage{textcomp}%
\usepackage{manyfoot}%
\usepackage{booktabs}%
\usepackage{algorithm}%
\usepackage{algorithmicx}%
\usepackage{algpseudocode}%
\usepackage{listings}%

\newcommand{\ra}{\rightarrow}
\newcommand{\lra}{\leftrightarrow}
\newcommand{\Ra}{\Rightarrow}
\newcommand{\sbe}{\subseteq}
\newcommand{\sms}{\setminus}

\newcommand{\al}{\alpha}
\newcommand{\be}{\beta}
\newcommand{\ph}{\varphi}

\theoremstyle{thmstyleone}%

\theoremstyle{thmstyletwo}%

\theoremstyle{thmstylethree}%
\newtheorem{definition}{Definition}%
\newtheorem{proposition}{Proposition}%
\newtheorem{theorem}{Theorem}
\begin{document}

\title[Article Title]{Solution of the Hempel's statistical ambiguity problem and Causal AI}


\author*[1]{\fnm{Evgenii} \sur{Vityaev}}\email{vityaev@math.nsc.ru}

\affil*[1]{\orgname{Sobolev Institute of Mathematics of the SB RAS}, \orgaddress{\street{Koptuga 4}, \city{Novosibirsk}, \postcode{630090}, \country{Russia}}}


\abstract{This paper addresses Carl Hempel's longstanding problem of statistical ambiguity in inductive-statistical inference, in which contradictory predictions are derived from statistical laws. To avoid such predictions, Carl Hempel proposed the Requirement of Maximal Specificity (RMS) for the statistical laws used in the inference. An analysis of the RMS refinements made by Wesley Salmon, Alberto Coffa, and James Fetzer led to the following definition of maximally specific statistical laws: "the lawlike premises of an adequate explanation must specify all and only those properties whose presence or absence made a difference to the occurrence of its explanandum-phenomenon." However, there was no proof of a solution to the statistical ambiguity problem based on this definition. We use Nancy Cartwright's definition of causes that raise probabilities across background contexts, and then introduce the concept of Causal Rules. Then we define a special semantic probabilistic inference procedure that incrementally refines these causal rules by incorporating all statistically relevant information. This procedure yields Maximally Specific Causal Relationships (MSCRs), for which we prove (Theorem 1) that predictions derived from them are consistent. This resolves the statistical ambiguity problem. The semantic probabilistic inference procedure provides a probabilistic causal learning system, which may be used in such new areas as Causal AI and Causal Machine Learning. They fundamentally explore causal inference as a tool for understanding cause-and-effect relationships within complex systems. Properties similar to RMS remain under discussion. Several notions related to RMS are considered: invariant feature learning, invariant causal prediction, and spurious association.}


\maketitle

\section{Introduction}

In recent years, within the fields of Causal AI and Causal Machine Learning, causal inference has become a critical tool for understanding cause-and-effect relationships. Integrating these relationships into machine learning models enables the creation of causal models for a deeper understanding of real-world systems. This paradigm provides more robust, interpretable, and actionable insights in areas such as medicine, finance, and autonomous systems.

To develop the probabilistic causal learning system that infers effects without contradictions we need to solve Hempel's \textit{statistical ambiguity problem}. It concerns the problem that contradictory predictions derived from inductive-statistical inference \cite{Hempel65, Hempel68}. To avoid such contradictions, Carl Hempel introduced the Requirement of Maximal Specificity (RMS) for statistical laws, which means that maximally specific statistical laws must incorporate all information relevant to prediction.
Subsequent analysis of RMS revealed numerous problems and disputes (see next section on Historical background). As a result, a solution to the "statistical ambiguity" problem was not reached, and the consistency of predictions for maximally specific statistical laws was not proven, and the problem remained unsolved. 

This article resumes the consideration of the "statistical ambiguity problem" and RMS. We define RMS based on Hempel's definition and define the notion of maximally specific cause-effect relationships (MSCR). Then we prove a theorem (Theorem 1) that predictions based on MSCRs are consistent. This solves Hempel's "statistical ambiguity problem" for maximally specific cause-effect relationships.
For the works in Causal AI and Causal Machine Learning it is rather important. 

Based on Cartwright's definition of causes that raise  probabilities of their effects in various background contexts \cite{Cartwright, Stanford} we define a stronger notion of Causal Rules (CR).  

\begin{itemize}
\item[] Cartwright: C causes E if and only if \\ $P(E\mid C\&B)>P(E\mid\neg{C}\&B)$ \\ for every background context B.
\medskip
\item[] The background in the Causal Rules consists of all other conditions of the rule.
\medskip
\item[] Causal Rule (see Definition 6): The rule $A_1\&\dots\&A_k \Rightarrow A_0$, \\ is a \textit{causal rule} iff every literal $A_1,\dots,A_k$ is a cause of $A_0$ \\
$P(A_0\mid A_1\&\dots\&A_k) > P(A_0\mid \neg{A_i}\&B)$, $i=1,\dots,k$, \\ relative to background $B=\&\{\{A_1,\dots,A_k\}\setminus A_i\}$, containing all literals except $A_i$. 
\end{itemize}

We introduce a  refinement procedure  for causal rules that incrementally strengthens the conditional probabilities of causal rules by adding all relevant information until the causal rules can no longer be refined, thus producing Maximally Specific Causal Relationships. It provides the learning method (Definition 8) for MSCR discovery. 
MSCRs may be considered as causal models, as they identify the key variables that underpin reliable cause-and-effect relationships. This method thus produces a causal learning system.

Let $\sf MSR$ is the set of all MSCRs. It is proved in Theorem 1 that I-S inference of predictions based on any subset of $\sf MSR$ is consistent. The set $\sf MSR$ thus provides a consistent set of probabilistic knowledge. The set $\sf MSR$ may be considered as a \textit{consistent probabilistic theory} in the same sense as a logical theory that is consistent and infers consistent conclusions.

The problem of statistical ambiguity for statistical laws of the form $A_1\&\dots\&A_k \Rightarrow A_0$, where $A_1,\dots,A_k,A_0$ are literals was considered in \cite{Author1, Author12}, and for the form $\phi \Rightarrow \psi$, where $\phi$ and $\psi$ are propositional formulas was considered in \cite{Author13}.

Our paper is organized as follows. Section 2 (RMS history) is devoted to the historical analysis of RMS and reveals discussions around it. Section 3 considers relation of RMS to Causal AI and Causal Machine Learning. Section 4 is devoted to definition of probability on propositional formulas. Section~5 considers rules, their probabilities and the definition of MSCR. In Section~6, we define I-S inference as a prediction operator for $\sf MSR$ and prove that applying it to a consistent set of MSCRs produces a consistent set. Finally, Section~7 is conclusion. 

\section{RMS history}

The problem of statistical ambiguity in inductive-statistical (I-S) explanation was first identified by Hempel. It generated a substantial literature attempting to articulate and refine the conditions under which probabilistic explanations can avoid contradictory conclusions. This section traces the evolution of the Requirement of Maximal Specificity (RMS) from its origins in Hempel's work through its subsequent critiques and reformulations, culminating in the identification of causal factors as the essential element for resolving the ambiguity.

\textbf{The Principle of Total Evidence and Its Limitations}.
Before addressing the Requirement of Maximal Specificity directly, it is essential to distinguish it from another fundamental principle of probabilistic inference: the requirement of total evidence. As articulated by Carnap and others, the requirement of total evidence states that in any rational application of probabilistic inference, the probability assigned to a hypothesis must be determined by reference to all available evidence. In Hempel's formulation, this principle directs that when assessing the credibility of a statement, one must consider the entire body of knowledge K: ``What degree of belief, or what probability, is it rational to assign to the statement 'Gi' in a given knowledge situation?" \cite{Hempel68}.

Initially, Hempel had conflated the requirement of maximal specificity with the requirement of total evidence, suggesting that RMS served as a "rough substitute" for the latter. However, in his 1968 reappraisal, he explicitly retracted this view, recognizing that the two principles address fundamentally different questions. The requirement of total evidence concerns the rational credibility of a statement based on all available information. The Requirement of Maximal Specificity, by contrast, addresses the explanatory status of an argument: it specifies conditions under which two sentences in $K$ — a statistical law $p(G,F)=r$ and a singular sentence $F_i$ — can serve to explain, relative to $K$, why $i$ is $G$. As Hempel emphasizes: The point of an explanation is not to provide evidence for the occurrence of the explanandum phenomenon, but to exhibit it as nomically expectable. And the probability attached to an $I-S$ explanation is the probability of the conclusion relative to the explanatory premises, not relative to the total class $K$. Thus, the requirement of total evidence simply does not apply to the determination of the probability associated with an $I-S$ explanation, and the requirement of maximal specificity is not "a rough substitute for the requirement of total evidence." \cite{Hempel68}

This clarification established RMS as an independent principle with its own distinct rationale, specifically designed to address the problem of explanatory ambiguity in statistical contexts.

\textbf{The Problem of Statistical Ambiguity and the Initial Formulation of RMS}.
The problem that RMS was designed to solve emerges from the nature of inductive-statistical explanation itself. As \cite{Hempel65} demonstrated, an $I-S$ argument has the form:

\medskip
\begin{tabular}{l} \label{*}
$p(G;F) = r$ \\
$F(a)$ \\
\hline
$G(a)$
\end{tabular}
\medskip\noindent

\noindent where line indicates an inductive relationship, and r represents the inductive probability of the explanandum given the explanans. 

Statistical ambiguity arises when two arguments with true premises yield contradictory conclusions. It can be illustrates with a classic example (cited Salmon). 

Suppose that we have the following statements.
\begin{itemize}
\item L1 Almost all cases of streptococcus infection clear up quickly after the administration of penicillin.
\item L2 Almost no cases of penicillin resistant streptococcus infection clear up quickly after the administration of penicillin.
\item C1 Jane Jones had streptococcus infection.
\item C2 Jane Jones received treatment with penicillin.
\item C3 Jane Jones had a penicillin resistant streptococcus infection.
\end{itemize}

Following the above pattern of I-S explanation it is possible to construct two contradictory arguments based on these statements. On the base of $L_1$ and $C1 \land C2$ one can explain why Jane Jones recovered quickly (E). The second argument with premises $L_2$ and $C2 \land C3$ explains why Jane Jones did not ($\lnot E$). The premises of both arguments are consistent with each other, they could all be true. The probability of both argument may be close to 1.

To block such conflicting explanations, Hempel introduced the Requirement of Maximal Specificity \cite{Hempel65}.

\begin{quote}
\textbf{RMS}: For an $I-S$ argument to be acceptable relative to a knowledge state $K$, for any predicate $H$ such that $K$ contains both $\forall x(H(x)\Rightarrow F(x))$ and $H(a)$, there must exist a statistical law $p(G;H)=r$ in $K$
with the same probability $r$. 
\end{quote}

The basic idea is that if $H$ provides more specific information about the object than $F$, then the law based on $H$ should be preferred — and if that law has a different probability, the original argument is invalid.

Hempel hoped to solve this problem by forcing all statistical laws in an argument to be maximally specific. That is, they should contain all relevant information with respect to the domain in question. In our example, then, the  premise $C3$ invalidates the first argument, since it is not maximally specific with respect to all information about Jane Jones. So, we can only explain $\neg$E, but not E.

\textbf{The Problem of Epistemic Relativity}. Coffa criticizes Hempel's insistence on relativizing $I-S$ explanation to a knowledge state K \cite{Coffa}.
Coffa begins by distinguishing between epistemic and non-epistemic concepts. A concept is epistemic if its meaning cannot be given without reference to knowledge; non-epistemic concepts—such as ``table," ``chair," or ``truth" — can be characterized independently of any knower. Hempel's thesis, Coffa explains, is that inductive explanation is not merely epistemic in the sense that it depends on what we know about the evidence (as in the case of ``well-confirmed D-N explanation"), but rather that it is a ``non-conformational epistemic concept". This means, as Hempel himself acknowledges, that ``there is no concept that stands to his epistemically relativized notion of inductive explanation as the concept of true D-N explanations stands to that of well-confirmed D-N explanation" \cite{Coffa}. In other words:
According to the thesis of epistemic relativity there is no meaningful notion of true inductive explanation. Hence, we could not possibly have reasons to believe that anything is a ``true inductive explanation". \cite{Coffa}.

Coffa's fundamental objection is that this conclusion is not merely surprising but fatal to the project of inductive explanation. If there is no notion of true inductive explanation, then $I-S$ explanations relativized to $K$ cannot be understood as arguments that we have reason to believe are explanations in the same sense that $D-N$ explanations are. 

More importantly, Coffa argues that Hempel has misidentified the nature of the problem. The real difficulty is not epistemic but ontological — it is the old problem of the reference class in a new guise.

We would like to suggest that when Hempel turned his attention to the theory of inductive explanation what he stumbled upon was the fact that the problem of defining a model of inductive explanation for single events was the other side of the coin of the single case problem. He stumbled, that is, upon the reference-class problem \cite{Coffa}.

The reference-class problem arises because a single event can belong to multiple reference classes with different probabilities for the outcome of interest. Jones belongs to the class of persons with streptococcus infection, the class of persons treated with penicillin, and the class of persons with penicillin-resistant infections. Each yields a different probability for recovery. The frequentist tradition, Coffa notes, held that it is "strictly meaningless to assign a probability to a single event" \cite{Coffa} because all reference classes are, in principle, equally legitimate.

Hempel's epistemic relativization attempts to solve this problem by appeal to knowledge: the appropriate reference class is the most specific class to which the individual is known to belong. But Coffa observes that this strategy has the ironic consequence that ``it is ignorance, rather than knowledge, that makes the maximal specificity principle look like a workable demand" \cite{Coffa}. As knowledge increases, the principle becomes increasingly difficult to satisfy; an omniscient being would find no inductive explanations at all.

Coffa's proposed alternative points toward an ontic rather than epistemic formulation: the requirement should refer to ``all relevant aspects of the explanandum" \cite{Coffa}, where relevance is understood not as statistical correlation but as nomic connection — ``a predicate being nomically relevant to another when a law of nature determines that changes in the first one generate changes in the second one" \cite{Coffa}. This suggestion anticipates later developments in causal approaches to explanation.

\textbf{Salmon's Statistical Relevance Approach and shift to Ontic Homogeneity}. 
A different line of critique emerged from Wesley Salmon, who questioned not merely the formulation of RMS but its fundamental orientation. Salmon argued that Hempel's requirement of high probability is misguided; what matters for explanation is not the magnitude of probability but the presence of statistical relevance relations.

The following example illustrates this point \cite{Salmon}:

\begin{description} 
\item {\em John Jones was almost certain to recover from his cold within a week, because he took vitamin C, and almost all colds clear up within a week after administration of vitamin C.}
\end{description}

The difficulty with this example is that colds tend to clear up within a week regardless of the medication administered, and controlled tests indicate that the percentage of recoveries is unaffected by the use of vitamin C.

Thus, even though the argument satisfies Hempel’s requirements — a high probability, a statistical law, and true premises — it fails to be explanatory because the putative explanatory factor (taking vitamin C) is irrelevant to the outcome. What is needed, Salmon argues, is not high probability but relevance — the fact that taking vitamin C makes a difference to the probability of recovery.

This led Salmon to develop the Statistical-Relevance (S-R) model, in which an explanation consists of partitioning a reference class into cells that are homogeneous with respect to the explanandum property, together with information about which cell contains the individual in question. Crucially, Salmon’s homogeneity requirement is objective rather than epistemic: a class is homogeneous if no further partition can be made that is relevant to the occurrence of the explanandum property, regardless of whether such partitions are known. This represents a fundamental shift from Hempel’s epistemic relativization to an ontic conception of explanation. As Salmon later put it, ``the identification of the appropriate explanans is fully objective” once the explanandum has been unambiguously specified (Salmon, 1989).

\textbf{Fetzer's Requirement of Strict Maximal Specificity and the Turn to Causation}. James Fetzer (\cite{Fetzer81, Fetzer93}) carries the ontic turn further, arguing that neither Hempel's epistemic RMS nor Salmon's statistical relevance conditions are sufficient. What is required is not merely statistical homogeneity but causal relevance. Fetzer introduces the Requirement of Strict Maximal Specificity (RSMS), which demands ``that the lawlike premises of an adequate explanation must specify all and only those properties whose presence or absence made a difference to the occurrence of its explanandum-phenomenon" \cite{Fetzer93}.

Fetzer's analysis reveals that the problem of statistical ambiguity cannot be resolved at the level of statistical relations alone. The fundamental question is not which reference class yields the highest probability or even which partition yields homogeneous cells, but rather: which factors are genuinely responsible for the outcome? This is an ontological question about causal structure, not an epistemological question about the choice of reference class.

\textbf{Conclusion}.
The evolution of the requirements for maximum specificity reveals the depth of the Hempel statistical ambiguity problem.  Salmon and Fetzer came close enough to solving the problem, but did not solve it. Their research is taken into account in our definition of "causal rule". Salmon's statistical significance is taken into account in the definition of "causal rules" as Cartwright's definition of  cause. James Fetzer's requirement of strict maximal specificity is taken into account by determining the statistical significance of each condition in the premise of the rule. Fetzer's requirement that "adequate explanations apply to all and only those properties whose presence influenced the explanation" is taken into account in the procedure for clarifying causal rules, which gradually strengthens the conditional probabilities of causal rules by adding all relevant information until the causal rules can no longer be clarified. This procedure allows us to obtain the Maximally Specific Causal Relationships that solve the problem of statistical ambiguity (Theorem 1).

\section{Causal AI and Causal Machine Learning}

In recent years, causal inference has emerged as a critical tool for understanding cause-and-effect relationships within complex systems. By incorporating causal reasoning into machine learning, models can move to deeper understanding of the real-world systems. 

The main theoretical achievement of MSCRs is the formal proof of consistency for predictions.
This approach shares the ambition of contemporary causal approaches to identify genuinely operative causes.

Properties similar to RMS remain under discussion \cite{Kaddour}. There are some notions related to RMS in the Causal Machine Learning \cite{Kaddour}:  
\begin{itemize}
\item \textit{Invariant Feature Learning}. The task of identifying features of our data X that are predictive of Y across a range of environments E. This definition is very similar to the N. Cartwright definition of causes that raise probabilities of their effects in various background contexts. It is included in our definition of causal rules. 
\item \textit{Invariant Causal Prediction} -- an algorithm to find the causal feature set, the minimal set of features which are causal predictors of a target variable \cite{Kaddour}. 
In our definition of causal rule this property is also fulfilled. 
\item \textit{Spurious association}. In the training dataset, pictures of cows typically exhibit alpine pasture backgrounds, a spurious association caused by the cow's natural habitat. Definitions of cause by N. Cartwright and a causal rule exclude these spurious associations. 
\end{itemize}

Thus, the introduced notion of ``causal rule" remains an orientation for other Causal AI and Causal Machine Learning methods to extract precise knowledge from data predicting without contradictions.

\section{Logic and Probability background}

Let $F(At)$ be a set of well-formed formulas constructed from a set of atoms $At$ using connectives $\&$, $\vee$, $\ra$, $\neg$. Equivalence $\lra$ is an abbreviation, $\ph\lra\psi=(\ph\ra\psi)\&(\psi\ra\ph)$. We define $\top$ as $\ph\vee\neg\ph$, where $\ph$ is some fixed formula. 

The conjunction of a finite set of formulas $T$ is  denoted by $\bigwedge T$. $V(\ph)$ denotes the set of atoms occurring in formula $\ph$. The algebra of formulas $\mathcal{F}(At)$ is a set $F(At)$ with naturally interpreted connectives. 

\begin{definition}
The models are mappings from $At$ to truth values $\{ 0,1\}$. Such mappings we call ($At$-){\em valuations}. Every valuation $val:At\ra \{ 0,1\}$ extends in a standard way to the set $F(At)$ as $val: F(At)\ra \{ 0,1\}$.
\end{definition}

Let $\mathfrak{G}$ be a set of valations.  A formula $\ph$ is said to {\em satisfiable in} $\mathfrak{G}$  if $val(\ph)=1$ for some $val\in\mathfrak{G}$. We say that $\ph$  {\em holds on} $\mathfrak{G}$ (and write $\mathfrak{G}\models\ph$) if $val(\ph)=1$ for all $val\in\mathfrak{G}$. Finally, a set $T\sbe F(At)$ is  said to be $\mathfrak{G}$-{\em consistent} if there is $val\in\mathfrak{G}$ such that $val(\ph)=1$ for all $\ph\in T$. The set of all valuations we denote $\mathfrak{A}$. By a logical tautology we mean a formula that holds on $\mathfrak{A}$.

For a set $\mathfrak{G}$ of valuations, the relation $\ph\equiv_{\mathfrak{G}}\psi$ is defined by $\ph\lra\psi$ holding on $\mathfrak{G}$. Then $\equiv_{\mathfrak{G}}$ is a congruence on $\mathcal{F}(At)$, the respective quotient is denoted as $\mathcal{B}^{\mathfrak{G}}(At)$. The coset of $\ph$ w.r.t. $\equiv_{\mathfrak{G}}$ is denoted as $[\ph]_{\mathfrak{G}}$ and the universe of $\mathcal{B}^{\mathfrak{G}}(At)$ equals $\{ [\ph]_{\mathfrak{G}}\mid \ph\in F(At)\}$. The operations of $\mathcal{B}^{\mathfrak{G}}(At)$ are denoted as $\&_{\mathfrak{G}}$, $\vee_{\mathfrak{G}}$, $\ra_{\mathfrak{G}}$, $\neg_{\mathfrak{G}}$ and the lattice order as $\sqsubseteq_{\mathfrak{G}}$. Recall that $[\ph]_{\mathfrak{G}}\sqsubseteq_{\mathfrak{G}} [\psi]_{\mathfrak{G}}$ iff $[\ph]_{\mathfrak{G}}= [\ph\&\psi]_{\mathfrak{G}}= [\ph]_{\mathfrak{G}}\&_{\mathfrak{G}}[\psi]_{\mathfrak{G}}$.

\begin{definition}
Finitely additive measure $\mu$ is defined on $\mathfrak{G}$ as $\mu:2^{\mathfrak{G}}\ra[0,1]$:
\begin{enumerate}
\item $\mu(\mathfrak{G})=1$, $\mu(\varnothing)=0$; 
\item $\mu(A_1\cup\ldots\cup A_n)=\mu(A_1)+ \ldots +\mu(A_n)$ for pairwise disjoint subsets $A_1$, \ldots , $A_n\sbe\mathfrak{G}$. 
\item $\mu(A)=0$ iff $A=\varnothing$
\end{enumerate}
\end{definition}

Elements of $\mathfrak{G}$ may be interpreted as outcomes of experiments. Further, we assume that only essential experiments are included in $\mathfrak{G}$, which explains why $\mu(\{ val\})\neq 0$ for all $val\in\mathfrak{G}$.

For every $\ph\in F(At)$, we  define $\ph^{\mathfrak{G}}:=\{ val\in\mathfrak{G}\mid val(\ph)=1\}$ and function $\nu : F(At) \ra \{ 0,1\}$ by the rule $\nu(\ph)=\mu(\ph^{\mathfrak{G}})$. Then the function $\nu$ satisfies the following properties.

\begin{proposition}
\begin{enumerate} 
\item $\nu(\ph)=1$ iff $\ph$ holds on $\mathfrak{G}$.
\item $\nu(\ph)=0$  iff $\ph$ is not satisfiable on $\mathfrak{G}$.
\item $\nu(\ph\vee\psi)=\nu(\ph)+\nu(\psi)$ iff $\ph\&\psi$ is not satisfiable in $\mathfrak{G}$.
\end{enumerate}
\end{proposition}

Thus we defined the probability on the set of propositional formulas in the sense of \cite{Fagin}.

\section{Method. Causal Rules and Semantic Probabilistic Inference}

By a {\em rule} we mean a syntactic object of the form
\[ r = \ph \Ra \psi,\] where $\ph, \psi\in F(At)$.
Formula $\ph$ is called a {\em body} of the rule, whereas $\psi$ is a {\em head} of the rule: $\ph=B(r)$ and $\psi=H(r)$. A rule $r$ cannot be identified with the implication $\ph\ra\psi$ because the probability of $r$ will be defined in a different way. Namely, for a rule $r = \ph \Ra \psi$, whose body is satisfiable on $\mathfrak{G}$, i.e., $\nu(\ph)\neq 0$, we put
\[\nu(r):=\nu(\psi | \ph)=\frac{\nu(\psi\&\ph)}{\nu(\ph)}.\]
In case $B(r)$ is not satisfiable on $\mathfrak{G}$, the value $\nu(r)$ remains undefined.

Notice that the value $\nu(r)$ was defined so that it is smaller than the probability of implication.

\begin{proposition} For every rule $r$ with $B(r)^{\mathfrak{G}}\neq\mathfrak{G}$, we have
\[\nu(r)\leq \nu(B(r)\ra H(r)).\]
Moreover, the equality $\nu(r)= \nu(B(r)\ra H(r))$ is equivalent to  $B(r)^{\mathfrak{G}}\subseteq H(r)^{\mathfrak{G}}$, i.e., to the fact that the implication $B(r)\ra H(r)$ holds on $\mathfrak{G}$.
\end{proposition}

This statement can be proved analogously to Theorem 2 from \cite{Author13}

\begin{definition}  Let $r_1$ and $r_2$ be two rules with the same head, $H(r_1)=H(r_2)$. We call $r_1$ a {\em specification} of $r_2$, symbolically $r_1\preccurlyeq r_2$, if $B(r_1)^{\mathfrak{G}}\subseteq B(r_2)^{\mathfrak{G}}$; rule $r_1$ is a {\em proper} specification of $r_2$, $r_1\prec r_2$, if  $B(r_1)^{\mathfrak{G}}\subsetneqq B(r_2)^{\mathfrak{G}}$. We say in this case that $r_2$ is a {\em (proper) generalization} of $r_1$.
\end{definition}

In other words, one of the two rules with the same head is a proper generalization of the other if its body is weaker from the logical point of view.

\begin{definition} Let rule $r_1$ be a specification of $r_2$. We say that $r_1$ is a {\em refinement} of $r_2$, symbolically $r_1 > r_2$, if $\nu(r_1)>\nu(r_2)$.
\end{definition}

Evidently, the relation $r_1 > r_2$ implies that $r_1$ is a proper specification of $r_2$.

\begin{definition}
A set $\mathcal{R}$  of rules is said to be {\em rich} if for every $r\in \mathcal{R}$ and an arbitrary $s$ such that $s> r$, there is $r'\in\mathcal{R}$ with $r'\preccurlyeq s$ and $\nu(r')>\nu(r)$.
\end{definition}

\begin{definition} Let $\mathcal{R}$ be a rich set of rules. Rule $r$ is a {\em causal rule relative to $\mathcal{R}$}, if $r\in\mathcal{R}$ and $r$ is a refinement of all its proper generalizations from $\mathcal{R}$. We assume that if 
$r\in\mathcal{R}$ and $(\top\Rightarrow H(r))\succ r$\ then \ $\nu(r)> \nu(\top\Rightarrow H(r))$.
\end{definition}

\begin{proposition}
\label{problaw}  Let $\mathcal{R}$ be a rich set of rules over a finite set $At$ of atoms. For every rule $r\in\mathcal{R}$, there exists its generalization $r'$ such that $r'$ is a causal rule relative to $\mathcal{R}$ and $\nu(r')\geq\nu(r)$.
\end{proposition}

\begin{proof}
Let $r=\ph\Rightarrow\psi\in\mathcal{R}$. Consider the set $$\Delta=\{ [\al]_{\mathfrak{G}}\mid \nu(\al\Rightarrow\psi)\geq\nu(r),\ [\al]_{\mathfrak{G}}\sqsubseteq_{\mathfrak{G}}[\ph]_{\mathfrak{G}}\}.$$ 
Recall that the condition $[\al]_{\mathfrak{G}}\sqsubseteq_{\mathfrak{G}}[\ph]_{\mathfrak{G}}$ means exactly that $\al\Rightarrow\psi\preccurlyeq r$. Since $\Delta$ is finite as a subset of $\mathcal{B}^{\mathfrak{G}}(At)$, we can choose an element $[\be]_{\mathfrak{G}}\in\Delta$ minimal w.r.t. $\sqsubseteq_{\mathfrak{G}}$. Since $\mathcal{R}$ is reach and $\be\Rightarrow\psi\preccurlyeq r$, there is $r'\in\mathcal{R}$ with $r'\preccurlyeq \be\Rightarrow\psi$. Assume that $r'$ is not a causal rule, then there is  a proper generalization $\gamma\Rightarrow\psi$ of $r'$ such that  $\nu(\gamma\Rightarrow\psi)\geq\nu(r')$.  Naturally, in this case $[\al]_{\mathfrak{G}}\in\Delta$ and since a proper generalization $\gamma\Rightarrow\psi$ is a proper generalization of $r'$, we have $[\gamma]_{\mathfrak{G}}\sqsubset_{\mathfrak{G}}[\be]_{\mathfrak{G}}$, which contradicts the minimality of $[\be]_{\mathfrak{G}}$. Thus, $r'$ is the required causal rule.
\end{proof}

\begin{definition}
\label{strongproblaw} A causal rule $r$ is  {\em strong relative to} $\mathcal{R}$, if there is no causal rule $r'$ such that $r'$ is a refinement of $r$.
\end{definition}

\begin{definition}
{\em Semantic Probabilistic Inference} (SPI) of $\psi$ is a sequence $r_1,\dots,r_k$ of causal rules $r_i\in\mathcal{R}$ with the head $\psi$ such that:
\begin{itemize} 
\item[] $\nu(r_1)> \nu(\top\Rightarrow \psi);$
\item[] $r_{i+1} > r_i$, i=1,\dots,k-1;
\item[] $r_k$ -- strong relative to $\mathcal{R}$.
\end{itemize}
\end{definition}

Finally, among all strong causal rules with a given head we distinguish rules with maximal probability.

\begin{definition} Let $\psi\in F(At)$. A strong causal rule $r$ with head $\psi$ is called a {\em maximal specific causal rule for} $\psi$ {\em relative to} $\mathcal{R}$, if its probability $\nu(r)$ is greatest among all strong causal rules with head $\psi$.
\end{definition}

We will say that $r$ is a maximal specific causal rule relative to $\mathcal{R}$ if $r$ is a maximal specific causal rule relative to $\mathcal{R}$ for $H(r)$. The set of all maximal specific causal rules we denote as {\sf MSCR}. Rules from {\sf MSCR} we consider as satisfying the Requirement of Maximal Specificity, because a specification of such rules does not lead to an increase of probability, which means that their bodies contains all statistically relevant information for the prediction of its head.

\begin{proposition} \label{mslaw} Let $At$ be finite. For every rule $r$ such that $H(r)=\psi$ and the value $\nu(r)$ is defined, there exists a maximal specific causal rule $r'$  for  $\psi$ such that  $\nu(r')\geq\nu(r)$.
\begin{proof}
 According to Proposition~\ref{problaw} the set $$\Delta=\{ [\al]_{\mathfrak{G}} \mid \al\Rightarrow\psi\ \mbox{is a probabilistic causal rule} \}$$ is non-empty. It follows immediately from Definition~\ref{strongproblaw} that $\al\Rightarrow\psi$ is a strong causal rule iff $[\al]_{\mathfrak{G}}$ is a minimal element of $\Delta$ w.r.t. $\sqsubseteq_{\mathfrak{G}}$. Since $\Delta$ is finite, the set of its $\sqsubseteq_{\mathfrak{G}}$-minimal elements is a finite non-empty set. So we can choose in this set an element $[\be]_{\mathfrak{G}}$ with the greatest value $\nu(\be\Rightarrow\psi)$. This is a required maximal specific causal rule $r'$  for  $\psi$. That $\nu(\be\Rightarrow\psi)\geq\nu(r)$ follows again from Proposition~\ref{problaw}.
\end{proof}
\end{proposition}

\section{Result}

Let $At$, a set $\mathfrak{G}$ of models, and a measure $\mu$ on $\mathfrak{G}$ be fixed. 
The set of all maximal specific causal rules relative to $\mathcal{R}$ in that case we denote as ${\sf MSCR}(\mathcal{R},\mathfrak{G},\mu)$. For the set of rules $\Pi\sbe{\sf MSCR}(\mathcal{R},\mathfrak{G},\mu)$ and $T\sbe F(At)$ we define an operator of {\em direct predictions}:

\[ Pr_{\Pi}(T)=T\cup \{ H(r) \mid  r\in\Pi, 
\exists \ph_1, ... ,\ph_n\in T (\mathfrak{G}\models (\ph_1\&\ldots\&\ph_n)\lra B(r))\}\]

Further, we put:
\[Pr^0_{\Pi}(T)=T,\ Pr^{n+1}_{\Pi}(T)=Pr_{\Pi}(Pr^{n}_{\Pi}(T)),\]
\[PR_{\Pi}(T)=\bigcup_{n\in\omega} Pr^{n}_{\Pi}(T). \]
We call  $PR_{\Pi}$ a {\em prediction operator for} $\Pi$.

\begin{theorem} Let $At$ be a finite set of atoms, $\Pi\sbe{\sf MSCR}(\mathcal{R},\mathfrak{G},\mu)$, and $T\sbe F(At)$. If $T$ is $\mathfrak{G}$-consistent, then $PR_{\Pi}(T)$ is $\mathfrak{G}$-consistent too.

\begin{proof} Obviously, it will be enough to check that the operator of direct predictions produces a $\mathfrak{G}$-consistent set of formulas. First of all we show that the set $T$ of formulas and the set $\Pi$ of rules can be replaced by finite sets. Consider the family of cosets $\{ [\ph]_{\mathfrak{G}} \mid \ph\in T\}$, which is finite since $At$ is finite. For every $[\ph]_{\mathfrak{G}}$ from this family choose a single representative and put it into $T'$. Now we consider the set of pairs:
\[\Theta=\{ ([\ph]_{\mathfrak{G}}, [\psi]_{\mathfrak{G}}) \mid \ph\Rightarrow\psi\in\Pi.\}\] A finite set of rules $\Pi'$ is defined as follows. For every pair of cosets $([\ph]_{\mathfrak{G}}, [\psi]_{\mathfrak{G}})\in\Theta$, we choose a single pair of representatives $(\ph,\psi)$ and put the rule $\ph\Rightarrow\psi$ into $\Pi'$. It is clear that for every $\psi\in Pr_{\Pi}(T)$ there is a $\psi'$ in $Pr_{\Pi'}(T')$ such that $\psi\equiv_{\mathfrak{G}}\psi'$. In this way, if $Pr_{\Pi'}(T')$ is  $\mathfrak{G}$-consistent, then $Pr_{\Pi}(T)$ is  $\mathfrak{G}$-consistent too. Further, let $\Pi''$ be the set of such rules $r$ from $\Pi'$ that the equivalence $(\ph_1\&\ldots\&\ph_n)\ra B(r)$ holds on $\mathfrak{G}$ for some $\ph_1$,\ldots, $\ph_n\in T'$.

Let $\Pi''=\{ r_1, \ldots, r_n\}$. We put $T_0=T'$, $T_{i+1}=T_i\cup\{ H(r_i)\}$. Clearly, $T_n=Pr_{\Pi'}(T')$. Using induction on $i$ we show that every $T_i$ is $\mathfrak{G}$-consistent.

Assume that $T_i$ is $\mathfrak{G}$-consistent, but $T_{i+1}$ is not. Let $r_i=\ph\Rightarrow\psi$. By definition of $\Pi''$ there are $\chi_1,\ldots,\chi_n\in T_i$ such that $(\chi_1\&\ldots\&\chi_n)\lra\ph$ holds on $\mathfrak{G}$. Let $N=T_i\sms\{ \chi_1,\ldots\chi_n\}$. Assume that $\{\ph, \neg(\bigwedge N)\}$ is $\mathfrak{G}$-consistent, i.e., $\nu(\ph\&\neg(\bigwedge N))\neq 0$. In this case for $s=\ph\& \neg(\bigwedge N)\Ra\psi$ we have:
\[\nu(s)=\frac{\nu(\ph\&\neg(\bigwedge N)\&\psi)}{\nu(\ph\&\neg(\bigwedge N))}=\frac{\nu(\ph\&\psi)-\nu(\ph\&\bigwedge N\&\psi)}{\nu(\ph)-\nu(\ph\&\bigwedge N)}.\]

 We have $\mathfrak{G}\models \bigwedge T_{i+1}\lra(\ph\&\bigwedge N\&\psi)$ and $\mathfrak{G}\models \bigwedge T_{i}\lra(\ph\&\bigwedge N)$ by choice of $\chi_1,\ldots,\chi_n$. Since by assumption $\nu(\bigwedge T_{i+1})=0$ and $\nu(\bigwedge T_{i})\neq 0$, we conclude that $\nu(\ph\&\bigwedge N\&\psi)=0$ and $\nu(\ph\&\bigwedge N)\neq 0$. In this way, we have
\[ \nu(s)=\frac{\nu(\ph\&\psi)}{\nu(\ph)-\nu(\ph\&\bigwedge N)} > \frac{\nu(\ph\&\psi)}{\nu(\ph)}=\nu(r_i).\]

Since  $ s\prec r_i$ and $\nu(s)>\nu(r_i)$, then there is $r'\in\mathcal{R}$ such that $r'\preceq s$ and $\nu(r')> \nu(r_i)$. On the other hand, from $r_i\in{\sf MSCR}(\mathcal{R},\mathfrak{G},\mu)$ and $r_i\succ r'$ we obtain $\nu(r_i)\geq \nu(r')$. This contradiction proves that the body of $s$ is not $\mathfrak{G}$-consistent: $\nu(\ph\&\neg(\bigwedge N))= 0$.  As a consequence we obtain $\nu(\ph\&\neg(\bigwedge N)\&\psi)= 0$. Now we have:
\[\nu(\ph\&\psi)=\nu(\ph\&\psi)- \nu(\ph\&\neg(\bigwedge N)\&\psi)=\nu(\ph\&\bigwedge N\&\psi)=0.\]
Thus, $\nu(r_i)=0$. At the same time $r_i$ is a causal rule, which implies $0=\nu(r_i) > \nu(\top\Rightarrow\psi)\geq 0$. The obtained contradiction concludes the proof.
\end{proof}
\end{theorem}

\section{Conclusion}
Analysis of the history surrounding the RMS discussions has made it possible to precisely formalize and resolve Carl Hempel’s statistical ambiguity problem. This result does not only sum up the historical debate about RMS requirements, but also may open new directions in the philosophy of science. It follows from the result that I-S inference can discover logically consistent natural-scientific theories. It also follows that cyclic causal relations, which reflect the integrity of objects of the perceived world, loop back on themselves and form consistent “causal models” of the external world categories \cite{Rehder, Rehder2}. Based on these causal models a ``probabilistic formal concepts” were defined that provide the idealized description of categories \cite{VitDemPon, Author12}. The causal relations between actions and their outcomes describe the goal-directed activity of humans and animals in accordance with the physiological Theory of Functional Systems \cite{Nadin}.

For a class of rules of the form $\alpha_1\wedge\ldots\wedge\alpha_n\Rightarrow\beta$, where $\alpha_i$ and $\beta$ are literals (atoms or their negations), a program system ``Discovery" was developed that discovers a set of MSCRs on a sample data D for prediction of some goal property $\psi$, using the exact Fisher test for statistical estimation of probabilistic inequalities of definition 8 \cite{mind}. This system was successfully applied to solve several tasks such as financial forecasting \cite{KovVit00}, medicine \cite{KovVitR01}, and bio informatics \cite{VitOVPK02}. 

Based on ${\sf MSCR}$ a more powerful program system may be developed for solving Causal AI tasks. For example, for the digital twins control systems development, where decisions based on predictions are rather important.

\bibliography{sn-bibliography}

\end{document}